\documentclass{article} 
\usepackage[accepted]{icml2024_for_arXiv}
\usepackage{url}
\usepackage[utf8]{inputenc} 
\usepackage[T1]{fontenc}    
\usepackage{hyperref}       
\usepackage{url}            
\usepackage{booktabs}       
\usepackage{amsfonts}       
\usepackage{nicefrac}       
\usepackage{microtype}      
\usepackage{xcolor}         
\usepackage{svg}
\usepackage{times}
\usepackage{epsfig}
\usepackage{graphicx}
\usepackage{amsmath}
\usepackage{amssymb}
\usepackage{url}            
\usepackage{booktabs}       
\usepackage{amsfonts}       
\usepackage{commath}
\usepackage{wrapfig}
\usepackage{mathtools,amsbsy}
\usepackage[inline]{enumitem} 
\usepackage[all]{foreign}
\usepackage{amsthm}
\usepackage{wrapfig}
\usepackage{tabularx}
\usepackage{pifont}
\newcommand{\cmark}{\ding{51}}%
\newcommand{\xmark}{\ding{55}}%
\usepackage{pgfplots}
\usepackage{tikz-3dplot}
\usepackage{tikzscale}
\usepackage{centeredtikzarrowbulletheads}
\pgfplotsset{%
    compat = 1.17,
    compat/show suggested version = false,
}
\usetikzlibrary{external,fit,matrix,backgrounds}
\DeclareRobustCommand{\includetikzgraphicsfile}[2][]{\tikzsetnextfilename{#2}\includegraphics[#1]{figures/#2.tikz}}
\DefineNamedColor{named}{myblue}{RGB}{68,114,196}

\newtheorem{theorem}{Theorem}
\newtheorem{proposition}[theorem]{Proposition}
\newtheorem{lemma}[theorem]{Lemma}
\newtheorem*{theorem*}{Theorem}

\newtheorem*{proposition*}{Proposition}

\DeclareMathOperator{\SO}{\textup{SO}}
\DeclareMathOperator{\Og}{\textup{O}}

\icmltitlerunning{O$n$ Learning Deep O($n$)-Equivariant Hyperspheres}

\begin{document}
\twocolumn[
\icmltitle{O$n$ Learning Deep O($n$)-Equivariant Hyperspheres}

\begin{icmlauthorlist}
\icmlauthor{Pavlo Melnyk}{yyy}
\icmlauthor{Michael Felsberg}{yyy}
\icmlauthor{Mårten Wadenbäck}{yyy}
\icmlauthor{Andreas Robinson}{yyy}
\icmlauthor{Cuong Le}{yyy}
\end{icmlauthorlist}

\icmlaffiliation{yyy}{Computer Vision Laboratory, Department of Electrical Engineering, Linköping University, Sweden}
\icmlcorrespondingauthor{Michael Felsberg}{michael.felsberg@liu.se}
\vskip 0.3in
]
\printAffiliationsAndNotice{}

\begin{abstract}
\textit{In this paper, we utilize hyperspheres and regular $n$-simplexes and propose an approach to learning deep features equivariant under the transformations of $n$D reflections and rotations, encompassed by the powerful group of $\Og(n)$.
Namely, we propose $\Og(n)$-equivariant neurons with spherical decision surfaces that generalize to any dimension $n$, which we call \texttt{Deep Equivariant Hyperspheres}.
We demonstrate how to combine them in a network that directly operates on the basis of the input points and propose an invariant operator based on the relation between two points and a sphere, which as we show, turns out to be a Gram matrix. 
Using synthetic and real-world data in $n$D, we experimentally verify our theoretical contributions and find that our approach is superior to the competing methods for $\Og(n)$-equivariant benchmark datasets (classification and regression), demonstrating a favorable speed/performance trade-off.
The code is available on \href{https://github.com/pavlo-melnyk/equivariant-hyperspheres}{GitHub}.}
\end{abstract}

\section{Introduction}
\label{introduction}
\textit{Spheres}\footnote{By \textit{sphere}, we generally refer to an $n$D sphere or a hypersphere; \eg, a circle is thus a 2D sphere.} serve as a foundational concept in Euclidean space while simultaneously embodying the essence of non-Euclidean geometry through their intrinsic curvature and non-linear nature. This motivated their usage as decision surfaces encompassed by spherical neurons \citep{perwass2003spherical, melnyk2020embed}.

Felix Klein's Erlangen program of 1872 \citep{hilbert_cohnvossen_1952} introduced a methodology to unify non-Euclidean geometries, emphasizing the importance of studying geometries through their invariance properties under transformation groups.
Similarly, geometric deep learning (GDL) as introduced by \citet{bronstein2017geometric, bronstein2021geometric} constitutes a unifying framework for various neural architectures.
This framework is built from the first principles of geometry---symmetry and scale separation---and enables tractable learning in high dimensions.
\begin{figure*}
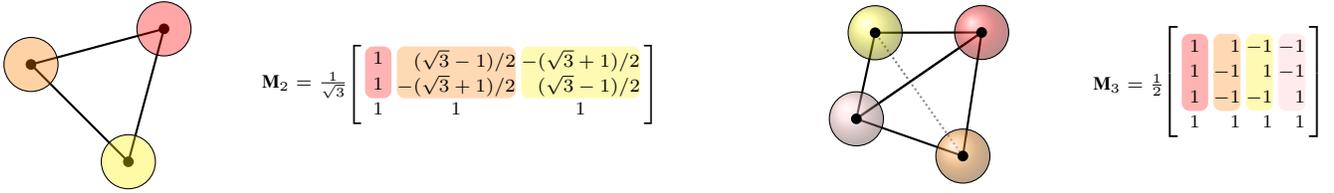

    \label{fig:main_figure}
    \includetikzgraphicsfile[width=\linewidth]{simplexes_and_M}
    \caption{The central components of \texttt{Deep Equivariant Hyperspheres} (best viewed in color): regular $n$-simplexes with the $n$D spherical decision surfaces located at their vertices and the simplex change-of-basis matrices $\textbf{M}_n$ (displayed for 
    $n=2$ and $n=3$).}
\end{figure*}

Symmetries play a vital role in preserving structural information of geometric data and allow models to adjust to different geometric transformations.
This flexibility ensures that models remain robust and accurate, even when the input data undergo various changes. 
In this context, spheres exhibit a maximal set of symmetries compared to other geometric entities in Euclidean space. 

The orthogonal group $\Og(n)$ fully encapsulates the symmetry of an $n$D sphere, including both rotational and reflection symmetries.
Integrating these symmetries into a model as an inductive bias is often a crucial requirement for problems in natural sciences and the respective applications, \eg, molecular analysis, protein design and assessment, or catalyst design \citep{rupp2012fast, ramakrishnan2014quantum, townshend2020atom3d, jing2021learning, lan2022adsorbml}.

In this paper, we consider data that live in Euclidean space (such as point clouds) and undergo rotations and reflections, \ie, transformations of the $\Og(n)$-group. 
Enriching the theory of steerable 3D spherical neurons \citep{melnyk2022steerable}, we present a method for learning $\Og(n)$-equivariant deep features using regular $n$-simplexes\footnote{We use the fact that a regular $n$-simplex contains $n+1$ equidistant vertices in $n$D.} and $n$D spheres, which we call \texttt{Deep Equivariant Hyperspheres} (see Figure~\ref{fig:main_figure}).
The name also captures the fact that the vertices of a regular $n$-simplex lie on an $n$D sphere, and that our results enable combining equivariant hyperspheres in multiple layers, thereby enabling $deep$ propagation.

Our main contributions are summarized as follows:
\begin{itemize}[topsep=0pt,itemsep=0pt]
    \item We propose O($n$)-equivariant spherical neurons, called \texttt{Deep Equivariant Hyperspheres}, readily generalizing to any dimension (see Section~\ref{sec:the_meat}). 
    \item We define and analyze generalized concepts for a network composed of the proposed neurons, including the invariant operator modeling the relation between two points and a sphere \eqref{eq:Delta}.
    \item Conducting experimental validation on both synthetic and real-world data in $n$D, we demonstrate the soundness of the developed theoretical framework, outperforming the related equi- and invariant methods and exhibiting a favorable speed/performance trade-off.
\end{itemize}

\section{Related work}
\label{sec:related_work}
The concept of spheres is also an essential part of spherical convolutional neural networks (CNNs) and CNNs designed to operate on 360 imagery \citep{Coors2018ECCV, su2017learning, Esteves_2018_ECCV, cohen2018spherical, perraudin2019deepsphere}.
In contrast, our method does not map input data \textit{on} a sphere, $\mathcal{S}^2$, nor does it perform convolution on a sphere. Instead, it embeds input in a higher-dimensional Euclidean space by means of a quadratic function.
Namely, our work extrapolates the ideas from prior work by \citet{perwass2003spherical, melnyk2020embed}, in which spherical decision surfaces and their symmetries have been utilized for constructing equivariant models for the 3D case \citep{melnyk2022steerable, melnyk2022tetrasphere}.
We carefully review these works in Section~\ref{sec:background}.

Similarly to the approach of \cite{ruhe2023clifford}, our \texttt{Deep Equivariant Hyperspheres} directly operate on the basis of the input points, not requiring constructing an alternative one, such as a steerable spherical harmonics basis.
Constructing an alternative basis is a key limitation of many related methods \citep{anderson2019cormorant,thomas2018tensor,fuchs2020se3}.
Our method also generalizes to the orthogonal group of any dimensionality.

Another type of method is such as by \citet{finzi2021practical}, a representation method building equivariant feature maps by computing an integral over the respective group (which is intractable for continuous Lie groups and hence, requires coarse approximation).
Another category includes methods operating on scalars and vectors: they update the vector information by learning the parameters conditioned on scalar information and multiplying the vectors with it \citep{satorras2021n, jing2021learning}, or by learning the latent equivariant features \citep{deng2021vector}.

While the methods operating on hand-crafted O($n$)-invariant features are generally not as relevant to our work, the methods proposed by \citet{villar2021scalars} and \citet{xu2021sgmnet} are: The central component of the scalars method \citep{villar2021scalars} is the computation of the pair-wise scalar product between the input points, just as the $\Og(3)$-invariant descriptor by \citet{xu2021sgmnet} is a sorted Gram matrix (SGM) of input point coordinates, which encodes global context using relative distances and angles between the points. 
In Section~\ref{sec:higher_order}, we show how this type of computation naturally arises when one considers the relation between two points and a sphere for computing invariant features in our approach, inspired by the work of \citet{li2001generalized}.

\section{Background}
\label{sec:background}
In this section, we present a comprehensive background on the theory of spherical neurons and their rotation-equivariant version, as well as on the general geometric concepts used in our work.
\subsection{Spherical neurons via non-linear embedding}
\label{sec:spherical_neurons}
Spherical neurons \citep{perwass2003spherical, melnyk2020embed} are neurons with, as the name suggests, spherical decision surfaces.
By virtue of conformal geometric algebra \citep{li2001generalized}, \citet{perwass2003spherical} proposed to embed the data vector $\textbf{x}\in\mathbb{R}^n$ and represent the sphere with center $\textbf{c} =(c_1, \dots, c_n)\in\mathbb{R}^n$ and radius $r\in\mathbb{R}$ respectively as
\begin{equation}
	\label{hypersphere_in_r}
	\begin{aligned}
		\textbf{\textit{X}} = \big(x_1, \dots, x_n, -1, -\frac{1}{2}\lVert\textbf{x}\rVert^2\big)\in\mathbb{R}^{n+2},\\ 
		\textbf{\textit{S}} = \big(c_1, \dots, c_n, \frac{1}{2}(\lVert\textbf{c}\rVert^2 - r^2), 1\big)\in\mathbb{R}^{n+2},
	\end{aligned}
\end{equation}
and used their scalar product $\textbf{\textit{X}}^\top \textbf{\textit{S}} = -\frac{1}{2}\norm{\textbf{x}-\textbf{c}}^2 + \frac{1}{2}r^2$
as a classifier, \ie, the spherical neuron:
\begin{equation}
    \label{eq:spherical_neuron}
     f_{S}(\textbf{\textit{X}}; \textbf{\textit{S}}) = \textbf{\textit{X}}^\top \textbf{\textit{S}},
 \end{equation}
with learnable parameters $\textbf{\textit{S}}\in \mathbb{R}^{n+2}$.

The sign of this scalar product depends on the position of the point $\textbf{x}$ relative to the sphere $(\textbf{c}, r)$: inside the sphere if positive, outside of the sphere if negative, and on the sphere if zero \citep{perwass2003spherical}. 
Geometrically, the activation of the spherical neuron \eqref{eq:spherical_neuron} determines the cathetus length of the right triangle formed by $\textbf{x}$, $\textbf{c}$, and the corresponding point on the sphere (see Figure~2~in~\citet{melnyk2020embed}).

We note that with respect to the data vector $\textbf{x}\in\mathbb{R}^n$, a spherical neuron represents a non-linear function $f_{S}(\,\cdot\,;\textbf{\textit{S}}): \mathbb{R}^{n+2} \rightarrow \mathbb{R}$, due to the inherent non-linearity of the embedding \eqref{hypersphere_in_r}, and therefore, does not necessarily require an activation function, as observed by \citet{melnyk2020embed}.
The components of $\textbf{\textit{S}}$ in \eqref{hypersphere_in_r} can be treated as \textit{independent} learnable parameters. 
In this case, a spherical neuron learns a \textit{non-normalized} sphere of the form $\widetilde{\textbf{\textit{S}}} = ({s_1}, \dots, {s_{n+2}}) \in \mathbb{R}^{n+2}$, which represents the same decision surface as its normalized counterpart defined in \eqref{hypersphere_in_r}, thanks to the homogeneity of the embedding \citep{perwass2003spherical, li2001generalized}.

\subsection{Equi- and invariance under $\Og(n)$-transformations}
\label{sec:equivariance}
The elements of the orthogonal group $\Og(n)$ can be represented as $n \times n$ matrices $\textbf{\textit{R}}$ with the properties $\textbf{\textit{R}}^\top \textbf{\textit{R}} = \textbf{\textit{R}} \textbf{\textit{R}}^\top = \textbf{I}_n$, where $\textbf{I}_n$ is the identity matrix, and $\det{\textbf{\textit{R}}} = \pm 1$, geometrically characterizing $n$D rotations and reflections. 
The special orthogonal group $\SO(n)$ is a subgroup of $\Og(n)$ and includes only orthogonal matrices with the positive determinant, representing rotations.

We say that a function $f : \mathcal{X} \rightarrow \mathcal{Y}$ is $\Og(n)$-equivariant if for every $\textbf{\textit{R}} \in \Og(n)$ there exists the transformation representation, $\rho(\textbf{\textit{R}})$, in the function output space, $\mathcal{Y}$, such that
\begin{equation}
\label{eq:equivariance}
    \rho(\textbf{\textit{R}}) \, f(\textbf{\textup{x}}) = f(\textbf{\textit{R}} \textbf{\textup{x}}) \text{\quad for all~} \textbf{\textit{R}} \in \Og(n), \; \textbf{\textup{x}} \in \mathcal{X} \subseteq \mathbb{R}^n.
\end{equation}
We call a function $f: \mathcal{X} \rightarrow \mathcal{Y}$ $\Og(n)$-invariant if for every  $\textbf{\textit{R}} \in \Og(n)$, $\rho(\textbf{\textit{R}}) = \textbf{I}_n$. That is, if
\begin{equation}
\label{eq:invariance}
    f(\textbf{\textup{x}}) = f(\textbf{\textit{R}} \textbf{\textup{x}}) \text{\quad for all~} \textbf{\textit{R}} \in \Og(n), \; \textbf{\textup{x}} \in \mathcal{X} \subseteq \mathbb{R}^n.
\end{equation}
Following the prior work convention \citep{melnyk2022steerable, melnyk2022tetrasphere} hereinafter, we write  $\textbf{\textit{R}}$ to denote the same $n$D rotation/reflection as an $n \times n$ matrix in the Euclidean space $\mathbb{R}^n$, as an $(n+1) \times (n+1)$ matrix in the projective (homogeneous) space $\textit{P}(\mathbb{R}^n) \subset \mathbb{R}^{n+1}$, and as an $(n+2) \times (n+2)$ matrix in $\mathbb{R}^{n+2}$.
For the latter two cases, we achieve this by appending ones to the diagonal of the original $n \times n$ matrix without changing the transformation itself \citep{melnyk2020embed}.
\subsection{Steerable 3D spherical neurons and TetraSphere}
\label{sec:steerable_3d_neurons}
Considering the 3D case, \citet{melnyk2022steerable} showed that a spherical neuron \citep{perwass2003spherical,melnyk2020embed} can be \textit{steered}.
In this context, \textit{steerability} is defined as the ability of a function to be written as a linear combination of the rotated versions of itself, called \textit{basis functions} \citep{freeman1991design, knutsson1992aframework}.
For details, see the Appendix (Section~\ref{sec:A_background}).

According to \citet{melnyk2022steerable}, a 3D steerable filter consisting of spherical neurons needs to comprise a \textit{minimum} of four 3D spheres: one learnable spherical decision surface $\textbf{\textit{S}} \in \mathbb{R}^5$ \eqref{hypersphere_in_r} and its three copies \textit{rotated} into the other three vertices of the regular tetrahedron, using one of the results of \citet{freeman1991design} that the basis functions must be distributed in the space uniformly. 

To construct such a filter, \ie, a steerable 3D spherical neuron, the main (learned) sphere center $\textbf{c}_0$ needs to be rotated into $\frac{\norm{\textbf{c}_0}}{\sqrt{3}}\,(1,1,1)$ by the corresponding (geodesic) rotation $\textbf{\textit{R}}_O$. 
The resulting sphere center is then rotated into the other three vertices of the regular tetrahedron.
This is followed by rotating all four spheres back to the original coordinate system.
A steerable 3D spherical neuron can thus be defined by means of the $4 \times 5$ matrix $B(\textbf{\textit{S}})$ containing the four spheres:
\begin{equation}
	\label{eq:sphere_filter_bank}
        \textup{F}(\textbf{\textit{X}};\textbf{\textit{S}}) = B(\textbf{\textit{S}}) \textbf{\textit{X}}~,\quad
	B(\textbf{\textit{S}}) = 
	\begin{bmatrix}
		(\textbf{\textit{R}}_O^{\top}\, \textbf{\textit{R}}_{T_i}\, \textbf{\textit{R}}_O\, \textbf{\textit{S}})^\top\\
	\end{bmatrix}_{i=1}^{4}~,
\end{equation}
where  ${\textbf{\textit{X}}} \in \mathbb{R}^{5}$ is the input 3D point embedded using \eqref{hypersphere_in_r}, $\{\textbf{\textit{R}}_{T_i}\}_{i=1}^{4}$ is the $\mathbb{R}^5$ rotation isomorphism corresponding to the rotation from the first vertex, \ie, $(1, 1, 1)$ to the $i$-th vertex of the regular tetrahedron\footnote{Therefore, $\textbf{\textit{R}}_{T_1}=\textbf{I}_5$, \ie, the original $\textbf{\textit{S}}$ remains at $\textbf{c}_0$.}.

\citet{melnyk2022steerable} showed that steerable 3D spherical neurons are $\SO(3)$-equivariant (or more precisely, $\Og(3)$-equivariant, as remarked in \citet{melnyk2022tetrasphere}):
\begin{equation}
\label{eq:filter_bank_equivariance}
    V_{\textbf{\textit{R}}} \, B(\textbf{\textit{S}}) \, \textbf{\textit{X}} = B(\textbf{\textit{S}})\,\textbf{\textit{R}}\textit{\textbf{X}},\quad
    V_{\textbf{\textit{R}}} = \textup{\textbf{M}}^\top \textbf{\textit{R}}_O\, \textbf{\textit{R}}\, \textbf{\textit{R}}_O^{\top} \textup{\textbf{M}} ~,
\end{equation}
where $\textbf{\textit{R}}$ is a representation of the 3D rotation in {$\mathbb{R}^5$}, and $V_{\textbf{\textit{R}}} \in G < \SO(4)$ is the 3D rotation representation in the filter bank output space, with $\textup{\textbf{M}} \in \SO(4)$ being a change-of-basis matrix that holds the homogeneous coordinates of the tetrahedron vertices in its columns as 
\begin{equation}
    \label{eq:M3}
    \textbf{M}  =  
	\begin{bmatrix}
		\textbf{m}_1 & \textbf{m}_2 & \textbf{m}_3 & \textbf{m}_4	\end{bmatrix} = \frac{1}{2}
	\begin{bmatrix}
		1 &  \phantom{-}1 & -1             &  -1   \\
		1 & -1            &  \phantom{-}1  &  -1   \\
	    1 & -1            & -1             &  \phantom{-}1   \\
	    1 &  \phantom{-}1 &  \phantom{-}1  &  \phantom{-}1   \\
	\end{bmatrix}.
\end{equation}
We note that with respect to the input vector $\textbf{x}\in\mathbb{R}^3$, a steerable 3D spherical neuron represents a non-linear rotational-equivariant function $\textup{F}(\,\cdot\,;\textbf{\textit{S}}): \mathbb{R}^5 \rightarrow \mathbb{R}^4$ with the learnable parameters $\textbf{\textit{S}} \in \mathbb{R}^5$. 
\paragraph{TetraSphere} As the first reported attempt to \textit{learn} steerable 3D spherical neurons in an end-to-end approach, \citet{melnyk2022tetrasphere} has presently introduced an approach for $\Og(3)$-invariant point cloud classification based on said neurons and the VN-DGCNN architecture~\citep{deng2021vector}, called TetraSphere.

Given the point cloud input $\mathcal{X} \in \mathbb{R}^{N \times 3}$, the TetraSphere approach suggests to learn 4D features of each point by means of the TetraTransform layer $l_{\textup{TT}}(\,\cdot\,;\textbf{S}): \mathbb{R}^{N \times 3} \rightarrow \mathbb{R}^{N \times  4 \times K}$ that consists of $K$ steerable spherical neurons $B(\textbf{\textit{S}}_k)$ (see \eqref{eq:sphere_filter_bank}) that are shared among the points.
After the application of TetraTransform, pooling over the $K$ dimensions takes place, and the obtained feature map is then propagated through the VN-DGCNN network as-is.
However, the questions of how to combine the steerable neurons in multiple layers and how to make them process data in dimensions other than $3$ have remained open.

\subsection{Regular simplexes}
\label{sec:regular_simplexes}
Geometrically, a regular $n$-simplex represents $n+1$ equidistant points in $n$D \citep{elte2006semiregular}, lying on an $n$D sphere with unit radius. In the 2D case, the regular simplex is an equilateral triangle; in 3D, a regular tetrahedron, and so on.

Following \citet{cevikalp2023deep}, we compute the Cartesian coordinates of a regular $n$-simplex as $n+1$ vectors $\textbf{p}_i \in \mathbb{R}^n$:
\begin{equation}
\label{eq:simplex}
\begin{aligned} \textbf{p}_i &= \begin{cases} 
      n^{-1/2} \, \textup{\textbf{1}},  & i = 1 \\
      \kappa \, \textbf{1} + \mu \, \textbf{e}_{i-1}, & 2 \leq i \leq n+1 ~,
    \end{cases}\\
  \kappa &= -\frac{1+\sqrt{n+1}}{n^{3/2}}~,~~ \mu = \sqrt{1 + \frac{1}{n}}~,
  \end{aligned}\\
\end{equation}
where $\textbf{1} \in \mathbb{R}^n$ is a vector with all elements equal to 1 and $\textbf{e}_{i}$ is the natural basis vector with the $i$-th element equal to 1.

For the case $n=3$, we identify the following connection between \eqref{eq:M3} and \eqref{eq:simplex}: the columns of $\textbf{M}$, $\textbf{m}_i \in \mathbb{R}^4$, are the coordinates of the regular 3-simplex appended with a constant and normalized to unit length; that is,
\begin{equation}
\textbf{m}_i = \frac{1}{p} \begin{bmatrix} \textbf{p}_i \\  1/\sqrt{3} \end{bmatrix} \textup{with } p =  \left\lVert \begin{bmatrix} \textbf{p}_i \\ 1/\sqrt{3} \end{bmatrix} \right\rVert, 1 \leq i \leq 4.
\end{equation}

\section{Deep Equivariant Hyperspheres}
\label{sec:the_meat}
In this section, we provide a complete derivation of the proposed $\Og(n)$-equivariant neuron based on a learnable spherical decision surface and multiple transformed copies of it, as well as define and analyze generalized concepts of equivariant bias, non-linearities, and multi-layer setup. 

While it is intuitive that in higher dimensions one should use more copies (\ie, vertices) than in the 3D case \citep{melnyk2022steerable}, it is uncertain exactly how many are needed.
We hypothesize that the vertices should constitute a regular $n$-simplex ($n+1$ vertices) and rigorously prove it in this section.
\subsection{The simplex change of basis}
\label{sec:simplex_change_of_basis}
We generalize the change-of-basis matrix \eqref{eq:M3} to $n$D  by introducing $\textup{\textbf{M}}_n$, an $(n+1) \times (n+1)$ matrix holding in its columns the coordinates of the regular $n$-simplex appended with a constant and normalized to unit length:
\begin{equation}
	\label{eq:nd_basis_matrix}
	\textbf{M}_n  =  
	\begin{bmatrix}
		\textbf{m}_i \end{bmatrix}_{i=1}^{n+1},~~ \textbf{m}_i = \frac{1}{p} \begin{bmatrix}\textbf{p}_i \\  n^{-1/2}\end{bmatrix},~~p=  \left\lVert \begin{bmatrix} \textbf{p}_i \\ n^{-1/2} \end{bmatrix} \right\rVert, 
\end{equation}
where the norms $p$ are constant, since $\lVert\textbf{p}_i\rVert = \lVert\textbf{p}_j\rVert$ for all $i$ and $j$ by definition of a regular simplex.
\begin{proposition}
\label{pr:M_orthogonal}
    Let $\textup{\textbf{M}}_n$ be the-change-of-basis matrix defined in \eqref{eq:nd_basis_matrix}. Then $\textup{\textbf{M}}_n$ is an $(n+1)$D rotation or reflection, \ie, $\textup{\textbf{M}}_n \in \Og(n+1)$ (see Section~\ref{sec:A_numeric_instances} in the Appendix for numeric examples). 
\end{proposition}
\begin{proof}
We want to show that $\textbf{M}_n^\top \textbf{M}_n = \textbf{I}_{n+1}$, which will prove that $\textbf{M}_n$ is orthogonal.
The diagonal elements of $\textbf{M}_n^\top \textbf{M}_n$ are $\textbf{m}_i^\top \textbf{m}_i = \lVert\textbf{m}_i\rVert ^2 = 1$ since $\lVert \textbf{m}_i \rVert = 1$.
The off-diagonal elements are found as $\textbf{m}_i^\top \textbf{m}_j = (\textbf{p}_i^\top \textbf{p}_j + n^{-1}) / p^2$ for $i \neq j$, where $p$ is defined in \eqref{eq:nd_basis_matrix}.
Note that $\textbf{p}_i^\top \textbf{p}_j$ is the same for all $i$ and $j$ with $i \neq j$ since, by definition of a regular simplex, the vertices $\textbf{p}_i$ are spaced uniformly.
Note that $\textbf{p}_i^\top \textbf{p}_j = -n^{-1}$ for all $i \neq j$ by definition \eqref{eq:simplex}.
Hence, the off-diagonal elements of $\textbf{M}_n^\top \textbf{M}_n$ are zeros and $\textbf{M}_n^\top \textbf{M}_n = \textbf{I}_{n+1}$.
\end{proof}
Since $\textup{\textbf{M}}_n \in \Og(n+1)$, the sign of $\det{\textup{\textbf{M}}}_n$ is determined by the number of reflections required to form the transformation.
In the case of a regular $n$-simplex, the sign of the determinant depends on the parity of $n$ \textbf{and} the configuration of the simplex vertices. 
In our case, $\textup{\textbf{M}}_n$ is a rotation for odd $n$, \ie, $\det{\textup{\textbf{M}}}_n = 1$, and a reflection for even $n$. 
Consider, for example, the case $n=3$. The matrix $\textbf{M}_3$ shown in \eqref{eq:M3} has $\det{\textbf{M}_3} = 1$, thus, is a 4D rotation, as stated in Section~\ref{sec:steerable_3d_neurons}.

\begin{lemma}
\label{lem:lemma}
Let $\textbf{\textup{M}}_n$ be the change-of-basis matrix defined in \eqref{eq:nd_basis_matrix}, and $\textbf{\textup{P}}_n$ an $n \times (n+1)$ matrix holding the regular $n$-simplex vertices, $\textbf{\textup{p}}_i$, in its columns, and $p = \left\lVert \begin{bmatrix} \textbf{\textup{p}}_i \\ n^{-1/2} \end{bmatrix} \right\rVert$, as defined in \eqref{eq:nd_basis_matrix}. 
Then
\begin{equation}
    \label{eq:M_times_P^T}
     \textbf{\textup{M}}_n \textbf{\textup{P}}_n^\top = p \begin{bmatrix}
        \textbf{\textup{I}}_n\\
        \textbf{\textup{0}}^\top        
    \end{bmatrix}.
\end{equation}
\end{lemma}
\begin{proof}
We begin by elaborating on \eqref{eq:nd_basis_matrix}:
\begin{equation}
    \label{eq:M_and_P}
    \textbf{M}_n = \frac{1}{p}\begin{bmatrix}
        \textbf{P}_n\\ n^{-1/2} \, \textbf{1}^\top
    \end{bmatrix}.
\end{equation}
We note that the norms of the rows of $\textbf{P}_n$ are also equal to $p$ since $\textbf{M}_n \in \Og(n+1)$ (as per Proposition~\ref{pr:M_orthogonal}).
Recall that $\textbf{P}_n$ is centered at the origin, and, therefore, for a constant $a \in \mathbb{R}$ and a vector of ones $\textbf{1} \in \mathbb{R}^{n+1}$, we obtain $a\,\textbf{1}^\top \textbf{P}_n^\top = \textbf{0}^\top$.
Remembering that the product $\textbf{M}_n \textbf{P}_n^\top$ is between $\mathbb{R}^{n+1}$ vectors, we plug \eqref{eq:M_and_P} into the LHS of \eqref{eq:M_times_P^T} and obtain
\begin{equation}
    \textbf{M}_n \textbf{P}_n^\top = \frac{1}{p}\begin{bmatrix}
        \textbf{P}_n\\ n^{-1/2} \, \textbf{1}^\top
    \end{bmatrix} \, \textbf{P}_n^\top  = \frac{p^2}{p} \begin{bmatrix}
        \textbf{I}_n\\
        \textbf{0}^\top        
    \end{bmatrix} = p \begin{bmatrix}
        \textbf{I}_n\\
        \textbf{0}^\top        
    \end{bmatrix}.
\end{equation}
\end{proof}

\subsection{Equivariant \texorpdfstring{$n$D}{nD} spheres}
\label{sec:main_proofs}
In this section, we generalize steerable 3D spherical neurons reviewed in Section~\ref{sec:steerable_3d_neurons}.
We denote an equivariant $n$D-sphere neuron (an \textit{equivariant hypersphere}) by means of the $(n+1)\times (n+2)$ matrix $B_n(\textbf{\textit{S}})$ for the spherical decision surface $\textbf{\textit{S}}\in\mathbb{R}^{n+2}$ with center $\textbf{c}_0 \in \mathbb{R}^n$ and an $n$D input $\textbf{x}\in\mathbb{R}^n$ embedded as $\textbf{\textit{X}}\in\mathbb{R}^{n+2}$ as
\begin{equation}
\label{eq:sphere_NDfilter_bank}
 \begin{aligned}
\textup{\textbf{F}}_n(\textbf{\textit{X}};\textbf{\textit{S}}) &= B_n(\textbf{\textit{S}})\, \textbf{\textit{X}}~,\quad\\
	B_n(\textbf{\textit{S}}) &= 
	\begin{bmatrix}
		(\textbf{\textit{R}}_O^{\top}\, \textbf{\textit{R}}_{T_i}\, \textbf{\textit{R}}_O\, \textbf{\textit{S}})^\top \\
	\end{bmatrix}_{i=1}^{n+1} ~,
 \end{aligned}
\end{equation}
where $\{\textbf{\textit{R}}_{T_i}\}_{i=1}^{n+1}$ is the $\mathbb{R}^{n+2}$ rotation isomorphism corresponding to the rotation from the first vertex to the $i$-th vertex of the regular $n$-simplex, and $\textbf{\textit{R}}_O \in \SO(n)$ is the geodesic (shortest) rotation\footnote{In practice, we compute it utilizing the Householder (double-) reflection method, \eg, as described by \citet{book/2013/golub_van_loan}.} from the sphere center $\textbf{c}_0$ to $\norm{\textbf{c}_0}\textbf{p}_1$.
Therefore, $\textbf{\textit{R}}_{T_1}=\textbf{I}_{n+2}$.
(Technically, if the center $\textbf{c}_0$ happens to be $-\textbf{p}_1$, $\textbf{\textit{R}}_O$ is a reflection about the origin.
In principle, we could just as well write $\textbf{\textit{R}}_O \in$ O$(n)$, since it makes no difference in our further derivations.)

We now need to prove that $\textbf{\textup{F}}_n(\, \cdot \,;\textbf{\textit{S}})$ is $\Og(n)$-equivariant.
\begin{proposition}
Let $\textbf{\textup{F}}_n(\, \cdot \,;\textbf{\textit{S}}): \mathbb{R}^{n+2} \rightarrow \mathbb{R}^{n+1}$ be the neuron defined in \eqref{eq:sphere_NDfilter_bank} and $\textbf{R} \in \Og(n)$ be an $n \times n$ rotation or reflection matrix.
Then the transformation representation in the filter output space $\mathbb{R}^{n+1}$ is given by the $(n+1) \times (n+1)$ matrix
\begin{equation}
	\label{eq:V_n}
    V_n = \rho \left(\textbf{\textit{R}}\right)= \textup{\textbf{M}}_n^\top \textbf{\textit{R}}_O \, \textbf{\textit{R}}\, \textbf{\textit{R}}_O^{\top} \textup{\textbf{M}}_n ~,
\end{equation}
where $\textup{\textbf{M}}_n \in \textup{O}(n+1)$ is the-change-of-basis matrix defined in \eqref{eq:nd_basis_matrix} and a $1$ is appended to the diagonals of $\textbf{R}_O$ and $\textbf{R}$ to make them $(n+1) \times (n+1)$.
Furthermore, $V_n \in G < \Og (n+1)$.
\vspace{-3pt}
\end{proposition}
\begin{proof}
Since $\textup{\textbf{M}}_n \in \Og(n+1)$, $\textbf{\textit{R}}_O \in \SO(n)$, and $\textbf{\textit{R}} \in \Og(n)$ are orthogonal matrices, $V_n$ in \eqref{eq:V_n} is an orthogonal change-of-basis transformation that represents $\textbf{\textit{R}} \in \Og(n)$ in the basis constructed by $\textbf{M}_n$ and $\textbf{\textit{R}}_O$.
Note that appending one to the diagonal of $\textbf{\textit{R}}\in \Og(n)$ does not affect the sign of the determinant, which makes $V_n$ a reflection representation if $\det{\textbf{\textit{R}}} = -1$, or a rotation representation if $\det{\textbf{\textit{R}}} = +1$.
Since $\textbf{\textit{R}} \in \Og(n)$ and $\textbf{\textit{R}}_O \in \Og(n)$ embedded in the $(n+1)$D (by extending the matrix diagonal with $1$), not all elements of $\Og(n+1)$ can be generated by the operation in \eqref{eq:V_n}.
Thus, we conclude that $V_n$ belongs to a proper subgroup of $\Og(n+1) = \Og(n) \times \mathcal{S}^n$, \ie, $G < \Og (n+1)$, where $G$ is formed as $\Og(n) \times \textbf{M}_n \left[\frac{\textbf{c}_0}{\norm{\textbf{c}_0}}~~0\right]$ with $\textbf{M}_n \left[\frac{\textbf{c}_0}{\norm{\textbf{c}_0}}~~0\right] \in \mathcal{S}^n$.

The original transformation is found directly as
\begin{equation}
	\label{eq:R_to_V}
	\textbf{\textit{R}} = \textbf{\textit{R}}_O^{\top} \textbf{M}_n\, V_n \, \textbf{M}_n^{\top} \textbf{\textit{R}}_O~,
\end{equation}
followed by the retrieval of the upper-left $n \times n$ sub-matrix.
\end{proof}
Noteworthy, the basis determined by $\textbf{\textit{R}}_O\in \SO(n)$, which depends on the center $\textbf{c}_0$ of the sphere $\textbf{\textit{S}} \in \mathbb{R}^{n+2}$ (see \eqref{eq:sphere_NDfilter_bank}), will be different for different $\textbf{c}_0$. 
Therefore, the representation $V_n$ will differ as well.

\begin{theorem}
\label{th:the_theorem}
The neuron $\textup{\textbf{F}}_n(\, \cdot \,;\textbf{\textit{S}}): \mathbb{R}^{n+2} \rightarrow \mathbb{R}^{n+1}$ defined in \eqref{eq:sphere_NDfilter_bank} is $\Og(n)$-equivariant.
\end{theorem}
\begin{proof}
To prove the theorem, we need to show that \eqref{eq:equivariance} holds for $\textup{\textbf{F}}_n(\, \cdot \,;\textbf{\textit{S}})$.

 We substitute \eqref{eq:V_n} to the LHS and \eqref{eq:sphere_NDfilter_bank} to the RHS, and obtain
\begin{equation}
    \label{eq:nd_filter_equivariance}
     V_n \, B_n(\textbf{\textit{S}}) \, \textbf{\textit{X}} = B_n(\textbf{\textit{S}})\,\textbf{\textit{R}}\textit{\textbf{X}}~.
\end{equation}

For the complete proof, please see the Appendix (refer to Section~\ref{sec:A_proofs}).
\end{proof}
We note that with respect to the input vector $\textbf{x}\in\mathbb{R}^n$, the equivariant hypersphere $\textbf{\textup{F}}_n(\, \cdot \,;\textbf{\textit{S}}): \mathbb{R}^{n+2} \rightarrow \mathbb{R}^{n+1}$ represents a non-linear $\Og(n)$-equivariant function.
It is also worth mentioning that the \textit{sum} of the output $\textbf{\textit{Y}} = B_n(\textbf{\textit{S}}) \, \textbf{\textit{X}}$ is an $\Og(n)$-invariant scalar, \ie, the DC-component, due to the regular $n$-simplex construction.

This invariant part can be adjusted by adding a scalar \textit{bias} parameter to the output $\textbf{\textit{Y}}$.
The concept of bias is imperative for linear classifiers, but for spherical decision surfaces \citep{perwass2003spherical}, it is implicitly modeled by the embedding \eqref{hypersphere_in_r}. 
We note, however, that adding a scalar bias parameter, $b \in \mathbb{R}$ to the output of an equivariant hypersphere \eqref{eq:sphere_NDfilter_bank} respects $\Og(n)$-equivariance:
\begin{proposition}
Let $\textbf{Y}\in\mathbb{R}^{n+1}$ be the output of the $\Og(n)$-equivariant hypersphere $\textup{\textbf{F}}_n(\, \cdot \,;\textbf{\textit{S}}): \mathbb{R}^{n+2} \rightarrow \mathbb{R}^{n+1}$ \eqref{eq:sphere_NDfilter_bank} given the input $\textbf{X}\in \mathbb{R}^{n+2}$, and $b \in \mathbb{R}$ be a bias parameter.
Then $\textbf{Y}' = \textbf{Y} + b\,\textbf{\textup{1}}$, where $\textbf{\textup{1}}$ is the vector of ones in $\mathbb{R}^{n+1}$, is also $\Og(n)$-equivariant.
\end{proposition}
\begin{proof}
We need to show that \eqref{eq:nd_filter_equivariance} also holds when the bias $b$ is added.
First, we use $V_n$---the representation of $\textbf{\textit{R}}\in \Og(n)$ from \eqref{eq:V_n}---and the fact that $\textbf{\textit{R}}$ and $\textbf{\textit{R}}_O$ are both appended 1 to their main diagonal to make them $(n+1) \times (n+1)$.
Then $V_n \, \textbf{1} = \textup{\textbf{M}}_n^\top \textbf{\textit{R}}_O \, \textbf{\textit{R}}\, \textbf{\textit{R}}_O^{\top} \textup{\textbf{M}}_n \textbf{1} = \textup{\textbf{M}}_n^\top \textbf{\textit{R}}_O \, \textbf{\textit{R}}\, \textbf{\textit{R}}_O^{\top} \begin{bmatrix}\textbf{0}\\ p\,\sqrt{n} \end{bmatrix}= \textup{\textbf{M}}_n^\top \begin{bmatrix}\textbf{0}\\ p\,\sqrt{n} \end{bmatrix} = \textbf{1}$, where $p$ is a scalar defined in \eqref{eq:simplex}.
Since the bias $b$ is a scalar, we use that $V_n \, b\textbf{1} = b V_n \, \textbf{1} = b \textbf{1}$.
We now consider the left-hand side of \eqref{eq:nd_filter_equivariance}:
$V_n \, \textbf{\textit{Y}}' = V_n \, (\textbf{\textit{Y}} + b\textbf{1}) = V_n \, B_n(\textbf{\textit{S}}) \, \textbf{\textit{X}} + V_n \, b\textbf{1} = V_n \, B_n(\textbf{\textit{S}}) \, \textbf{\textit{X}} + b\textbf{1}$.
Plugging the equality \eqref{eq:nd_filter_equivariance} into the last equation, we complete the proof: $V_n \, B_n(\textbf{\textit{S}}) \, \textbf{\textit{X}} + b\textbf{1} = B_n(\textbf{\textit{S}})\,\textbf{\textit{R}} \textit{\textbf{X}} + b\textbf{1}$.
\end{proof}
This result allows us to increase the capacity of the equivariant hypersphere by adding the learnable parameter $b \in \mathbb{R}$.
In addition, note that all $V_n \in G < \Og (n+1)$ can be characterized by the fact that they all have an eigenvector equal to $\frac{1}{\sqrt{n+1}} \textbf{\textup{1}}$, where $\textbf{\textup{1}}$ is the vector of ones in $\mathbb{R}^{n+1}$.

\subsection{Normalization and additional non-linearity}
\label{sec:normalization}
An important practical consideration in deep learning is feature normalization \citep{ioffe2015batch, ba2016layer}.
We show how the activations of the equivariant hypersphere \eqref{eq:sphere_NDfilter_bank} can be normalized maintaining the equivariance:
\begin{proposition}
\label{prop:normalization}
Let $\textbf{Y}\in\mathbb{R}^{n+1}$ be the $\Og(n)$-equivariant output of the hypersphere filter \eqref{eq:sphere_NDfilter_bank}.
Then $\textbf{Y} / \lVert \textbf{Y} \rVert$, where $\lVert \textbf{Y} \rVert \in \mathbb{R}$, is also $\Og(n)$-equivariant.
\end{proposition}
\begin{proof}
Since $\textbf{\textit{Y}}$ is $\Og(n)$-equivariant, $\lVert \textbf{\textit{Y}} \rVert$ is $\Og(n)$-invariant (the length remains unchanged under $\Og(n)$-transformations).
Hence, $\textbf{\textit{Y}} / \lVert \textbf{\textit{Y}} \rVert$ is also $\Og(n)$-equivariant.  
\end{proof}

To increase the descriptive power of the proposed approach, we can add non-linearity to the normalization step, following \citet{ruhe2023clifford}:
\begin{equation}
    \label{eq:non-linear_norm}
    \textbf{\textit{Y}} \mapsto \frac{\textbf{\textit{Y}}}{\sigma(a)\, (\lVert\textbf{\textit{Y}}\rVert - 1) + 1},
\end{equation}
where $a \in \mathbb{R}$ is a learnable scalar and $\sigma(\cdot)$ is the sigmoid function.
\subsection{Extracting deep equivariant features}
\label{sec:deep_propagation}
We might want to propagate the equivariant output of $\textbf{F}_n$ \eqref{eq:sphere_NDfilter_bank}, $\textbf{\textit{Y}} = B_n(\textbf{\textit{S}}) \, \textbf{\textit{X}}$, through spherical decision surfaces while maintaining the equivariance properties. 
One way to achieve it is by using $(n+1)$D spheres, \ie, $\textbf{F}_{n+1}$, since the output $\textbf{\textit{Y}}\in\mathbb{R}^{n+1}$.
Thus, the results established in the previous section not only allow us to use the equivariant hyperspheres \eqref{eq:sphere_NDfilter_bank} for $n$D inputs but also to cascade them in multiple layers, thus propagating equivariant representations by successively incrementing the feature space dimensionality with a unit step, \ie, $n\textup{D} \rightarrow (n+1)\textup{D}$.

Consider, for example, the point cloud patch $\mathcal{X} = \{ \textbf{x}_i \} ^N_{i=1}$ consisting of the coordinates of $N$ points $\textbf{x}\in \mathbb{R}^n$ as the input signal, which we can also consider as the $N \times n$ matrix $\textbf{X}$.
Given the equivariant neuron $\textbf{F}_n(\, \cdot \,; \textbf{\textit{S}})$, a  \textit{cascaded} $n\textup{D} \rightarrow (n+1)\textup{D}$ feature extraction procedure using equivariant hyperspheres $\textbf{F}_n(\, \cdot \,; \textbf{\textit{S}})$ for the given output dimensionality $d$ (with $d > n$) can be defined as follows (at the first step, $\textbf{\textit{X}} \gets \textbf{x}$):
\begin{equation}
\label{eq:algorithm}
\begin{aligned}
 &\textbf{\textit{X}} \in\mathbb{R}^n \rightarrow \texttt{embed}(\texttt{normalize}(\textbf{\textit{X}})) \rightarrow \textbf{F}_n(\textbf{\textit{X}}; \textbf{\textit{S}}) \rightarrow \\ &\texttt{embed}(\texttt{normalize}(\textbf{\textit{X}}+b)) \rightarrow  \textbf{F}_{n+1}(\textbf{\textit{X}}; \textbf{\textit{S}}) \rightarrow \\ &\ldots \rightarrow \textbf{F}_{d}(\textbf{\textit{X}}; \textbf{\textit{S}}) \rightarrow \texttt{normalize}(\textbf{\textit{X}}+b) \rightarrow \textbf{\textit{X}} \in \mathbb{R}^{d}~, 
\end{aligned}
\end{equation}
where $\texttt{embed}$ is the embedding according to \eqref{hypersphere_in_r}, $\texttt{normalize}$ is the optional activation normalization (see Proposition~\ref{prop:normalization}), and $b$ is an optional scalar bias.
\begin{proposition}
\label{prop:cascading}
Given that all operations involved in the procedure~\ref{eq:algorithm} are $\Og(n)$-equivariant, its output will also be  $\Og(n)$-equivariant.
\end{proposition}
The proof is given in the Appendix (Section~\ref{sec:A_proofs}).

Thus, given $\mathcal{X}$ as input, the point-wise cascaded application with depth $d$ \eqref{eq:algorithm} produces the equivariant features $\mathcal{Y} = \{ \textit{\textbf{Y}}_i \} ^N_{i=1}$, $\textit{\textbf{Y}} \in \mathbb{R}^{n+d}$, which we can consider as the $N \times (n+d)$ matrix $\textbf{Y}$.

In this case, we considered the width of each layer in \eqref{eq:algorithm} to be 1, \ie, one equivariant hypersphere.
In practice and depending on the task, we often use $K_l$ equivariant hyperspheres per layer $l$, with suitable connectivity between subsequent layers.

\subsection{Modelling higher-order interactions}
\label{sec:higher_order}
\begin{figure*}
    \centering
    \includegraphics[width=\linewidth]{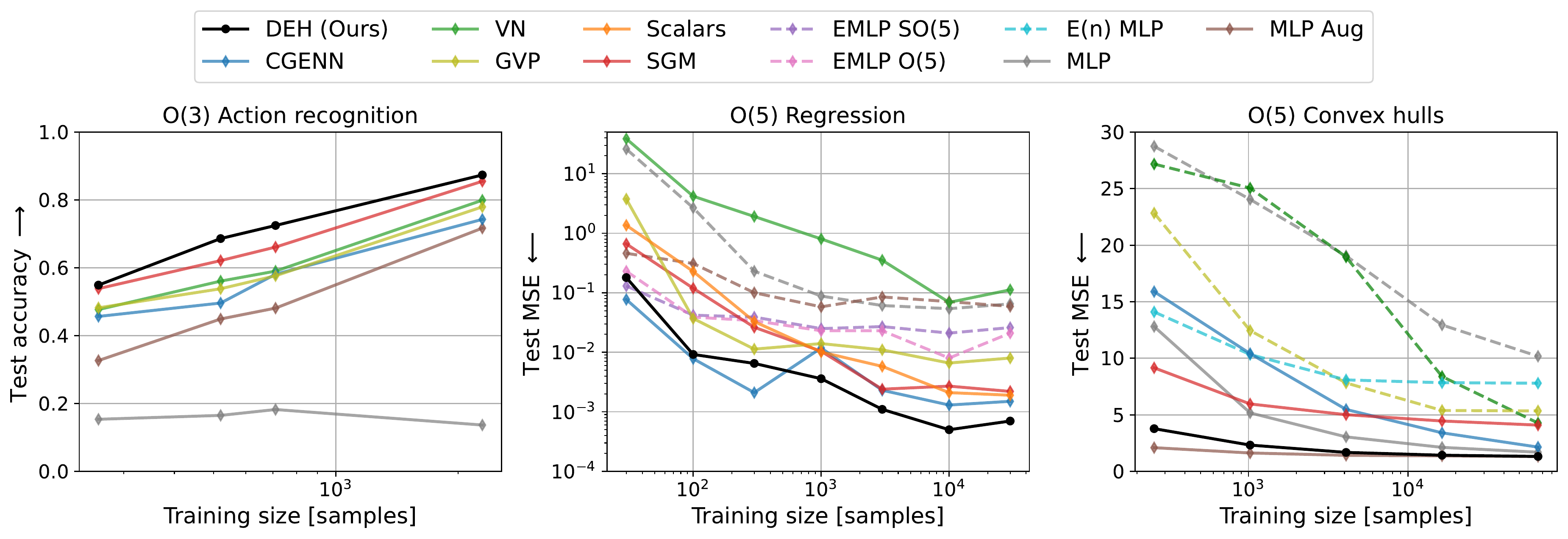}
    \caption{Left: real data experiment (the higher the accuracy the better); all the presented models are also permutation-invariant. Center and right: synthetic data experiments (the lower the mean squared error (MSE) the better); dotted lines mean that the results of the methods are copied from \citet{finzi2021practical} ($\Og(5)$ regression) or \citet{ruhe2023clifford} ($\Og(5)$ convex hulls). Best viewed in color.}
    \label{fig:O5-experiments}
\end{figure*}
The theoretical framework established above considers the interaction of \textit{one} point and one spherical decision surface, copied to construct the regular $n$-simplex constellation for the equivariant neuron in \eqref{eq:sphere_NDfilter_bank}.
To increase the expressiveness of a model comprised of equivariant hyperspheres, we propose to consider the relation of \textit{two} points and a sphere. 

Namely, following the work of \citet{li2001generalized}\footnote{See p. 22 in \citet{li2001generalized}.}, the relation between two points, $\textbf{x}_{1}$, $\textbf{x}_{2}$, and a sphere in $\mathbb{R}^{n}$, all embedded in $\mathbb{R}^{n+2}$ according to \eqref{hypersphere_in_r} as $\textbf{\textit{X}}_1$, $\textbf{\textit{X}}_2$, and $\textbf{\textit{S}}$, respectively, is formulated as
\begin{equation}
\label{eq:delta}
    \delta = \textit{e}_{12} \; \textbf{\textit{X}}_1^{\top} \textbf{\textit{S}} \; \textbf{\textit{X}}_2^{\top} \textbf{\textit{S}},
\end{equation}
where $\textit{e}_{12} := -\frac{1}{2} \, ({\lVert \textbf{x}_{1} - \textbf{x}_{2}\rVert}^2) \,\, \in \mathbb{R}$ models the edge as the squared Euclidean distance between the points.

To classify the relative position of the points to the sphere, we use the sign of $\delta \in \mathbb{R}$, and note that it is only determined by the respective sphere (scalar) activations, \ie, the scalar products $\textbf{\textit{X}}_i^\top \textbf{\textit{S}}$, since the edges $\textit{e}_{ij}$ are always negative. 
Thus, we may omit them, as we further demonstrate by the ablations in the Appendix (see Section~\ref{sec:A_ablations}).
Also, we note that in order to make $\delta$ an invariant quantity, we need to have \textit{equivariant} activations. 
Since the single sphere activations are not equivariant (see Section~\ref{sec:spherical_neurons}), we propose to substitute the single sphere $\textbf{\textit{S}}$ with our equivariant hyperspheres $B_n(\textbf{\textit{S}})$ \eqref{eq:sphere_NDfilter_bank}. 

Given the input $\textbf{X} \in \mathbb{R}^{N\times n}$ and the corresponding extracted equivariant features $\textbf{Y} \in \mathbb{R}^{N\times (n+d)}$, we compute 
\begin{equation}
\label{eq:Delta}
    \Delta = \textbf{X}^\top B_n(\textbf{\textit{S}}) \; (\textbf{X}^\top B_n(\textbf{\textit{S}}))^\top = \textbf{Y} \, \textbf{Y}^{\top}.
\end{equation}
The $\Og(n)$-invariance of $\Delta \in \mathbb{R}^{N \times N}$ follows from the fact that it is comprised of the 
Gram matrix $\textbf{Y} \, \textbf{Y}^{\top}$ that consists of the pair-wise inner products of equivariant features, which are invariant \citep{deng2021vector,melnyk2022tetrasphere}, just as in the case of directly computing the auto-product of the points \citep{xu2021sgmnet}. 
When permutation-invariance is desired, we achieve it by aggregation over the points, first following the procedure by \citet{xu2021sgmnet} and sorting the rows/columns of $\Delta$, and then applying max and/or mean pooling over $N$.
If multiple ($K_l$) equivariant hyperspheres per layer are used, \eqref{eq:Delta} is computed independently for all $K_l$, by computing $K_l$ Gram matrices, resulting in $\Delta \in \mathbb{R}^{N \times N \times K_l}$.

We show the effectiveness of the proposed invariant operator \eqref{eq:Delta} by the corresponding ablation in the Appendix (Section~\ref{sec:A_ablations}).
\vspace{-8pt}

\section{Experimental validation}
\label{sec:demonstration}
\vspace{-4pt}
In this section, we experimentally verify our theoretical results derived in Section~\ref{sec:the_meat} by evaluating our \texttt{Deep Equivariant Hyperspheres}, constituting feed-forward point-wise architectures, on real and synthetic $\Og(n)$-equivariant benchmarks.
In each experiment, we train the models using the same hyperparameters and present the test-set performance of the models chosen based on their validation-set performance.

For the sake of a fair comparison, all the models have approximately the same number of learnable parameters, and their final fully-connected layer part is the same.
A more detailed description of the used architectures is presented in the Appendix (see Section~\ref{sec:A_architectures}).
In addition to the performance comparison in Figure~\ref{fig:O5-experiments}, we compare the time complexity (\ie, the inference speed) of the considered methods\footnote{Some of the results in Figure~\ref{fig:O5-experiments} are copied from \citet{ruhe2023clifford} since the implementation of the specific versions of some models is currently unavailable. Therefore, we could not measure their inference speed.} in Figure~\ref{fig:trade-off}.
Furthermore, we present various ablations in the Appendix (see Section~\ref{sec:A_ablations}).
All the models are implemented in PyTorch~\citep{paszke2019pytorch}.

\subsection{$\Og(3)$ Action recognition}
\label{sec:o3-classification}
First, we test the ability of our method to utilize O(3)-equivariance as the inductive bias. 
For this experiment, we select the task of classifying the 3D skeleton data, presented and extracted by \citet{melnyk2022steerable} from the UTKinect-Action3D dataset by \citet{xia2012view}.
Each skeleton is a $20 \times 3$ point cloud, belonging to one of the $10$ action categories; refer to the work of \citet{melnyk2022steerable} for details.
We formulate the task to be both permutation- and O(3)-invariant.
\begin{figure*}
    \centering
    \includegraphics[width=\linewidth]{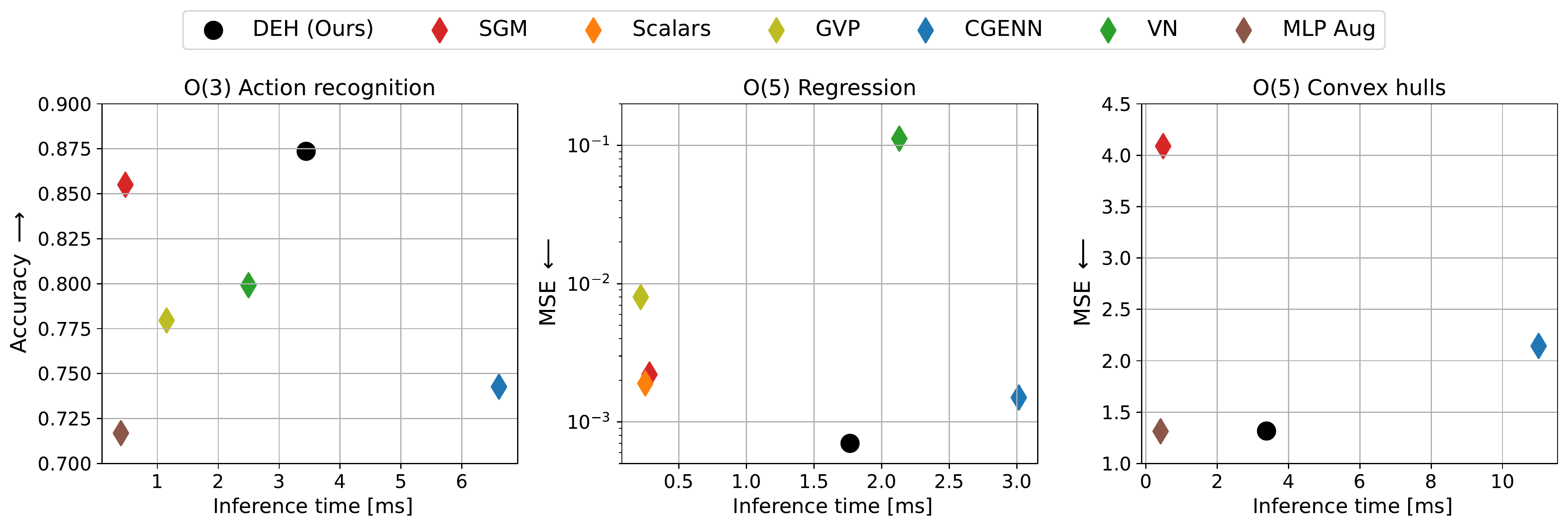}
    \caption{Speed/performance trade-off (the models are trained on all the available training data). Note that the desired trade-off is toward the top-left corner (higher accuracy and faster inference) in the left figure, and toward the bottom-left corner (lower error and faster inference) in the center and right figures. To measure inference time, we used an NVIDIA A100. Best viewed in color.}
    \label{fig:trade-off}
\end{figure*}

We construct an $\Og(3)$-equivariant point-wise feedforward model using layers with our equivariant hyperspheres (according to the blueprint of \eqref{eq:algorithm}) with the two-point interaction described in Section~\ref{sec:higher_order}, which we call DEH (see the illustration in Figure~\ref{fig:deh_architecture}). 
We also build a variant of the invariant SGM descriptor \citep{xu2021sgmnet} computing the Gram matrix of the input points, point-wise equivariant VN \citep{deng2021vector}, GVP \citep{jing2021learning}, and CGENN \citep{ruhe2023clifford} models and, as non-equivariant baselines, point-wise MLPs, in which the equivariant layers are substituted with regular non-linear ones.
We train one version of the baseline MLP with $\Og(3)$-augmentation, whereas our method is \emph{only trained on non-transformed skeletons}.

We evaluate the performance of the methods on the randomly $\Og(3)$-transformed test data.
The results are presented in Figure~\ref{fig:O5-experiments} (left): our DEH model, trained on the data in a single orientation, captures equivariant features that enable outperforming the non-equivariant baseline trained on the augmented data (MLP Aug). Moreover, DEH consistently outperforms the competing equivariant methods (VN, GVP, CGENN) and the invariant SGM model, demonstrating a favorable speed/performance trade-off, as seen in Figure~\ref{fig:trade-off} (left).

\subsection{$\Og(5)$ Regression}
\label{sec:o5-regression}
Originally introduced by \citet{finzi2021practical}, the task is to model the $\Og(5)$-invariant function $f(\textbf{x}_1, \textbf{x}_2) := \sin(\Vert \textbf{x}_1 \Vert) - \Vert \textbf{x}_2 \Vert^3 / 2 + \frac{\textbf{x}_1^\top \textbf{x}_2}{\Vert \textbf{x}_1 \Vert \Vert \textbf{x}_2 \Vert}$, where the two vectors $\textbf{x}_1 \in \mathbb{R}^5$ and $\textbf{x}_2 \in \mathbb{R}^5$ are sampled from a standard Gaussian distribution to construct train, validation, and test sets.
We use the same training hyperparameters and evaluation setup as \citet{ruhe2023clifford}.

Here, we employ a DEH architecture similar to that in Section~\ref{sec:o3-classification}, and compare it to the equivariant EMLPs \citep{finzi2021practical}, CGENN, VN, and GVP, and non-equivariant MLPs.
Refer to the Appendix (Section~\ref{sec:A_architectures}) for the architecture details.
Our results together with those of the related methods are presented in Figure~\ref{fig:O5-experiments} (center).
As we can see, our DEH is more stable than CGENN, as shown by the dependency over the training set size, and outperforms it in most cases. 
Our method also outperforms the vanilla MLP and the MLP trained with augmentation (MLP~Aug), as well as the $\Og(5)$- and $\SO(5)$-equivariant EMLP and VN, and the invariant SGM and Scalars \citet{villar2021scalars}\footnote{For this task, permutation-invariance is not required, and \citet{villar2021scalars} only used 3 unique elements of the Gram matrix constructed from the input. For the other two tasks, we refer to the results we obtained with SGM (where the entire Gram matrix with sorted rows is used).} methods.

\subsection{$\Og(5)$ Convex hull volume prediction}
Our third experiment addresses the task of estimating the volume of the convex hull generated by $16$ 5D points, described by \citet{ruhe2023clifford}.
The problem is $\Og(5)$-invariant in nature, \ie, rotations and reflections of a convex hull do not change its volume.
We exploit the same network architecture as in Section~\ref{sec:o3-classification} (see the Appendix for details).

We present our results alongside those of the related methods in Figure~\ref{fig:O5-experiments}: our DEH model outperforms all of the equi- and invariant competing methods (including an MLP version of the $\textup{E}(n)$-equivariant approach \citet{satorras2021n}) and in all the scenarios, additionally exhibiting a superior speed/performance trade-off, as seen in Figure~\ref{fig:trade-off} (right).
We outperform \citet{ruhe2023clifford} as well as their MLP result. However, our point-wise MLP implementation slightly outperforms our method on low-data regimes.

\section{Conclusion}
\label{sec:conslusion}
\vspace{-5pt}
In this manuscript, we presented \texttt{Deep Equivariant Hyperspheres} --- $n$D neurons based on spheres and regular $n$-simplexes --- equivariant under orthogonal transformations of dimension $n$.
We defined and analyzed generalized components for a network composed of the proposed neurons, such as equivariant bias, non-linearity, and multi-layer configuration (see Section~\ref{sec:the_meat} and the ablations in the Appendix).

In addition, we proposed the invariant operator \eqref{eq:Delta} modeling the relation between two points and a sphere, inspired by the work of \citet{li2001generalized}, and demonstrated its effectiveness (see the Appendix).
We evaluated our method on both synthetic and real-world data and demonstrated the utility of the developed theoretical framework in $n$D by outperforming the competing methods and achieving a favorable speed/performance trade-off (see Figure~\ref{fig:trade-off}).
Investigating the design of more advanced equivariant architectures of the proposed equivariant hyperspheres forms a clear direction for future work.

\section*{Limitations}
The focus of this paper is on the $\Og(n)$-equivariance, with $n > 3$, and our model architecture employed in the experiments was designed to compare with recent related methods, \eg, \citet{finzi2021practical, ruhe2023clifford}, using the small-scale $\Og(5)$ tasks presented in them.
To address scalability, one could wrap our equivariant hyperspheres in some kind of graph neural network (GNN) framework, \eg, DGCNN \citep{wang2019dgcnn}, thus utilizing the spheres in local neighborhoods.

Additionally, all the tasks considered in the experiments require invariant model output.
Therefore in the current version, we perform modeling of the interaction between the points using their $\Og(n)$-equivariant features (outputs of the proposed equivariant hyperspheres) only to produce invariant features \eqref{eq:Delta}. 
For tasks that require equivariant output, not considered in this work, this interaction needs to happen such that the equivariance is preserved (\eg, using Lemma 3 by \citet{villar2021scalars}).

Finally, if translation equivariance is additionally desired, \ie, the full $\textup{E}(n)$ group, a common way to address this is by centering the input point set by subtracting the mean vector, and then adding this vector to the model output.

\section*{Impact statements}
This paper presents work whose goal is to advance the field of Machine Learning. There are many potential societal consequences of our work, none of which we feel must be specifically highlighted here. An exception is possibly the area of molecular physics with applications in material science; the development of new materials might have a significant impact on sustainability.
\section*{Acknowledgments}
This work was supported by the Wallenberg AI, Autonomous Systems and Software Program (WASP), by the Swedish Research Council through a grant for the project Uncertainty-Aware Transformers for Regression Tasks in Computer Vision (2022-04266), and the strategic research environment ELLIIT. The computations were enabled by resources provided by the National Academic Infrastructure for Supercomputing in Sweden (NAISS) partially funded by the Swedish Research Council through grant agreement no. 2022-06725, and by the Berzelius resource provided by the Knut and Alice Wallenberg Foundation at the National Supercomputer Centre.

We highly appreciate the useful feedback from the reviewers. 
We thank Erik Darpö from the Department of Mathematics, Linköping University, for sharing his insights on group theory with us.

\bibliography{example_paper}
\bibliographystyle{icml2024}

\newpage
\appendix
\onecolumn
\section{Additional background}
\label{sec:A_background}

\subsection{Steerability}
\label{sec:A_steerability}
According to \citet{freeman1991design}, a function is called \textit{steerable} if it can be written as a linear combination of rotated versions of itself, as also alternatively presented by \citet{knutsson1992aframework}.
In 3D, $f^{\textbf{\textit{R}}}(x, y, z)$ is thus said to steer if
\begin{equation}
	\label{eq:3d_steerability_constraint}
	f^{\textbf{\textit{R}}}(x, y, z) = \sum_{j=1}^{M}v_j(\textbf{\textit{R}}) f^{\textbf{\textit{R}}_{j}}(x, y, z)~,
\end{equation}
where $f^{\textbf{\textit{R}}}(x, y, z)$ is $f(x, y, z)$ rotated by $\textbf{\textit{R}} \in \SO(3)$, and each $\textbf{\textit{R}}_j \in \SO(3)$ orients the corresponding $j$th basis function.

\citet{freeman1991design} further describe the conditions under which the 3D steerability constraint \eqref{eq:3d_steerability_constraint} holds and how to find the minimum number of basis functions, that must be uniformly distributed in space.

In this context, \citet{melnyk2022steerable} showed that in order to steer a spherical neuron defined in \eqref{eq:spherical_neuron} \citep{perwass2003spherical, melnyk2020embed}, one needs to have a minimum of fours basis functions, \ie, rotated versions of the original spherical neuron. 
This, together with the condition of the uniform distribution of the basis functions, leads to the regular tetrahedron construction of the steerable 3D spherical neuron in \eqref{eq:sphere_filter_bank}.

\section{Numeric instances for $n=\{2, 3, 4\}$}
\label{sec:A_numeric_instances}
To facilitate the reader's understanding of the algebraic manipulations in the next section, herein, we present numeric instances of the central components of our theory defined in \eqref{eq:simplex} and \eqref{eq:nd_basis_matrix}, for the cases $n=2$, $n=3$, and $n=4$.
For convenience, we write the vertices of the regular simplex \eqref{eq:simplex} as the $n\times (n+1)$ matrix $\textbf{P}_n = \begin{bmatrix}
    \textbf{p}_i
\end{bmatrix}_{i=1\ldots n+1}$.

\paragraph{$n=2:$} $\textbf{P}_2 = \frac{1}{\sqrt{2}} \begin{bmatrix}
1 & \phantom{-}(\sqrt{3}-1)/2 & -(\sqrt{3}+1)/2 \\
1 & -(\sqrt{3}+1)/2 & \phantom{-}(\sqrt{3}-1)/2 \\
\end{bmatrix}$, \quad\quad $p=\sqrt{3/2}$,
 
\paragraph{$\phantom{n=2:}$}
$\textbf{M}_2 = \frac{1}{\sqrt{3}} \begin{bmatrix}
            1 & \phantom{-}(\sqrt{3}-1)/2 & -(\sqrt{3}+1)/2 \\
            1 & -(\sqrt{3}+1)/2 & \phantom{-}(\sqrt{3}-1)/2 \\
            1 & 1 & 1 \\
\end{bmatrix}$.

\paragraph{$n=3:$}
$\textbf{P}_3 = \frac{1}{\sqrt{3}}
\begin{bmatrix}
1  & \phantom{-}1  & -1  & -1 \\
1  & -1  & \phantom{-}1  & -1 \\
1  & -1  & -1 & \phantom{-}1
\end{bmatrix}$,
\quad\quad\quad\quad\quad\quad\quad\quad\,$p=2/\sqrt{3}$,

\paragraph{$\phantom{n=3:}$}
$\textbf{M}_3 =\frac{1}{2}
	\begin{bmatrix}
		1 &  \phantom{-}1 & -1             &  -1   \\
		1 & -1            &  \phantom{-}1  &  -1   \\
	    1 & -1            & -1             &  \phantom{-}1   \\
	    1 &  \phantom{-}1 &  \phantom{-}1  &  \phantom{-}1   \\
	\end{bmatrix}$.
 
\paragraph{$n=4:$}
$\textbf{P}_4 = \frac{1}{2}
\begin{bmatrix}
1 & \phantom{\;}(3\sqrt{5} - 1)/4 & -(\sqrt{5} + 1)/4 & -(\sqrt{5} + 1)/4 & -(\sqrt{5} + 1)/4 \\
1 & -(\sqrt{5} + 1)/4 & \phantom{\;}(3\sqrt{5} - 1)/4 & -(\sqrt{5} + 1)/4 & -(\sqrt{5} + 1)/4 \\
1 & -(\sqrt{5} + 1)/4 & -(\sqrt{5} + 1)/4 & \phantom{\;}(3\sqrt{5} - 1)/4 & -(\sqrt{5} + 1)/4 \\
1 & -(\sqrt{5} + 1)/4 & -(\sqrt{5} + 1)/4 & -(\sqrt{5} + 1)/4 & \phantom{\;}(3\sqrt{5} - 1)/4
\end{bmatrix}$,

\paragraph{$\phantom{n=4:}$} \quad\quad\quad\quad\quad\quad\quad\quad\quad\quad\quad\quad\quad\quad\quad\quad\quad\quad\quad\quad \quad $p = \sqrt{5}/2$,

\paragraph{$\phantom{n=4:}$}
$\textbf{M}_4 = \frac{1}{\sqrt{5}}
\begin{bmatrix}
1 & \phantom{\;}(3\sqrt{5} - 1)/4 & -(\sqrt{5} + 1)/4 & -(\sqrt{5} + 1)/4 & -(\sqrt{5} + 1)/4 \\
1 & -(\sqrt{5} + 1)/4 & \phantom{\;}(3\sqrt{5} - 1)/4 & -(\sqrt{5} + 1)/4 & -(\sqrt{5} + 1)/4 \\
1 & -(\sqrt{5} + 1)/4 & -(\sqrt{5} + 1)/4 & \phantom{\;}(3\sqrt{5} - 1)/4 & -(\sqrt{5} + 1)/4 \\
1 & -(\sqrt{5} + 1)/4 & -(\sqrt{5} + 1)/4 & -(\sqrt{5} + 1)/4 & \phantom{\;}(3\sqrt{5} - 1)/4 \\
1 & 1 & 1 & 1 & 1
\end{bmatrix}$.

\section{Complete proofs}
\label{sec:A_proofs}
In this section, we provide complete proof of the propositions and theorems stated in the main paper.

\begin{theorem*}
\label{th:A_the_theorem}
(Restating Theorem~\ref{th:the_theorem}:)The neuron $\textup{\textbf{F}}_n(\, \cdot \,;\textbf{\textit{S}}): \mathbb{R}^{n+2} \rightarrow \mathbb{R}^{n+1}$ defined in \eqref{eq:sphere_NDfilter_bank} is $\Og(n)$-equivariant.
\end{theorem*}
\begin{proof}
We need to show that \eqref{eq:equivariance} holds for $\textup{\textbf{F}}_n(\, \cdot \,;\textbf{\textit{S}})$.

 We substitute \eqref{eq:V_n} to the LHS and \eqref{eq:sphere_NDfilter_bank} to the RHS, and obtain
\begin{equation}
    \label{eq:A_nd_filter_equivariance}
     V_n \, B_n(\textbf{\textit{S}}) \, \textbf{\textit{X}} = B_n(\textbf{\textit{S}})\,\textbf{\textit{R}}\textit{\textbf{X}}~.
\end{equation}

Keeping in mind that the $(n+1)$-th and $(n+2)$-th components, $s_{n+1}$ and $s_{n+2}$, of the sphere $\textbf{\textit{S}}\in \mathbb{R}^{n+2}$ with center $\textbf{c}_0 \in \mathbb{R}^n$ \eqref{hypersphere_in_r} are $\Og(n)$-invariant, as well as our convention on writing the rotation matrices (see the last paragraph of Section~\ref{sec:equivariance}), we rewrite the $(n+1) \times (n+2)$ matrix $B_n(\textbf{\textit{S}})$ using its definition \eqref{eq:sphere_NDfilter_bank}:
\begin{equation}
    \label{eq:B_n}
    B_n(\textbf{\textit{S}}) = 
	\begin{bmatrix}
		(\textbf{\textit{R}}_O^{\top}\, \textbf{\textit{R}}_{T_i}\, \textbf{\textit{R}}_O\, \textbf{\textit{S}})^\top\\
	\end{bmatrix}_{i=1}^{n+1} = 
        \begin{bmatrix}
	\textbf{c}_0^\top\textbf{\textit{R}}_O^\top\, \textbf{\textit{R}}_{T_i}^\top\, \textbf{\textit{R}}_O\ 
	    & s_{n+1} & s_{n+2}
	\end{bmatrix}_{i=1}^{n+1}.
\end{equation}

By definition of the rotation $\textbf{\textit{R}}_O$ \eqref{eq:sphere_NDfilter_bank}, we have that $\textbf{\textit{R}}_O\,\textbf{c}_0 = \lVert \textbf{c}_0 \rVert \textbf{p}_1$, where $\textbf{p}_1 \in \mathbb{R}^n$ is the first vertex of the regular simplex according to \eqref{eq:simplex}. 
Since $\textbf{\textit{R}}_{T_i}$ rotates $\textbf{p}_1$ into $\textbf{p}_i$, we obtain 
\begin{equation}
    \label{eq:c_and_P}
    \textbf{\textit{R}}_{T_i} \textbf{\textit{R}}_O\,\textbf{c}_0 = \lVert \textbf{c}_0 \rVert \textbf{p}_i~,\quad 1 \leq i \leq n+1~.
\end{equation}

Thus, we can write the RHS of \eqref{eq:A_nd_filter_equivariance} using the sphere definition \eqref{hypersphere_in_r} as
\begin{equation}
    \label{eq:RHS}
B_n(\textbf{\textit{S}})\,\textbf{\textit{R}}\textbf{\textit{X}} = 
        \begin{bmatrix}
	\lVert \textbf{c}_0 \rVert \textbf{p}_i^\top\, \textbf{\textit{R}}_O\ &
	    s_{n+1} & s_{n+2}
	\end{bmatrix}_{i=1}^{n+1} \textbf{\textit{R}}\textbf{\textit{X}} =
 \begin{bmatrix}
	\lVert \textbf{c}_0 \rVert \, \textbf{P}_n^\top\, \textbf{\textit{R}}_O \textbf{\textit{R}}\ & 
	    s_{n+1}\, \textbf{1} & s_{n+2}\, \textbf{1}   
   \end{bmatrix}\,\textbf{\textit{X}}.
\end{equation}
We now use the definition of $V_n$ from \eqref{eq:V_n} along with \eqref{eq:M_times_P^T}, \eqref{eq:M_and_P}, and \eqref{eq:c_and_P} to rewrite the LHS of \eqref{eq:A_nd_filter_equivariance} as
\begin{equation}
    \label{eq:LHS}
    \begin{aligned}
      V_n \, B_n(\textbf{\textit{S}}) \textbf{\textit{X}} &= \textbf{M}_n^{\top}\textbf{\textit{R}}_O\,\textbf{\textit{R}}\,\textbf{\textit{R}}_O^\top\,\textbf{M}_n \begin{bmatrix}
\begin{bmatrix}
    \lVert \textbf{c}_0 \rVert \, \textbf{P}_n^\top\,  & s_{n+1}\, \textbf{1}
\end{bmatrix}\textbf{\textit{R}}_O \ & s_{n+2}\, \textbf{1}
   \end{bmatrix}\,\textbf{\textit{X}} \\
   & = \textbf{M}_n^{\top}\textbf{\textit{R}}_O\,\textbf{\textit{R}}\,\textbf{\textit{R}}_O^\top\,\begin{bmatrix}
\begin{bmatrix}
    p\,\lVert \textbf{c}_0 \rVert
    \begin{bmatrix}
        \textbf{I}_n\\
        \textbf{0}^\top        
    \end{bmatrix}\,
    \begin{matrix}
        & \textbf{0} \\ & {p}\sqrt{n}\,s_{n+1} 
    \end{matrix}
\end{bmatrix}
    \textbf{\textit{R}}_O &
    \begin{matrix}
       \textbf{0} \\[1.5ex]
       {p}\sqrt{n}\,s_{n+2} \\
    \end{matrix}
   \end{bmatrix}\,\textbf{\textit{X}}\\
    & = \textbf{M}_n^{\top}\textbf{\textit{R}}_O\,\textbf{\textit{R}}\,
    \begin{bmatrix}
    \begin{bmatrix}
        p\,\lVert \textbf{c}_0 \rVert\ \, \textbf{I}_n, & \textbf{0} \\ \textbf{0}^\top & {p}\sqrt{n}\,
            s_{n+1} 
    \end{bmatrix}\textbf{\textit{R}}_O^\top\,\textbf{\textit{R}}_O & 
    \begin{matrix}
        \textbf{0} \\[1ex]
        {p}\sqrt{n}\, s_{n+2} \\
    \end{matrix}
   \end{bmatrix}\,\textbf{\textit{X}}\\
    & = 
    \frac{1}{p}\begin{bmatrix}
        \textbf{P}_n^\top\, \textbf{\textit{R}}_O\,\textbf{\textit{R}} & n^{-1/2} \textbf{1}
        \end{bmatrix}
    \begin{bmatrix}    p\,\lVert \textbf{c}_0 \rVert\, \textbf{I}_n& \textbf{0} & \textbf{0} \\
        \textbf{0}^\top & {p}\sqrt{n}\,
            s_{n+1} &{p}\sqrt{n}\, s_{n+2} 
   \end{bmatrix}\,\textbf{\textit{X}}\\
    & = 
    \begin{bmatrix}
        \,\lVert \textbf{c}_0 \rVert\, \textbf{P}_n^\top\, \textbf{\textit{R}}_O\,\textbf{\textit{R}}
      & \frac{\sqrt{n}}{\sqrt{n}}\, s_{n+1}\, \textbf{1} & \frac{\sqrt{n}}{\sqrt{n}}\, s_{n+2}\, \textbf{1}
    \end{bmatrix}
  \,\textbf{\textit{X}} \quad = \quad B_n(\textbf{\textit{S}})\,\textbf{\textit{R}}\textbf{\textit{X}}.
    \end{aligned}
\end{equation}
\end{proof}

\begin{proposition*}
(Restating Proposition~\ref{prop:cascading}:) Given that all operations involved in the procedure~\ref{eq:algorithm} are $\Og(n)$-equivariant, its output will also be  $\Og(n)$-equivariant.
\end{proposition*}
\begin{proof}
Let $\textbf{\textit{R}} \in \Og(n)$ be an orthogonal transformation, $\rho_i(\textbf{\textit{R}})$ the representation of $\textbf{\textit{R}}$ in the respective space, \eg, \eqref{eq:V_n} for the equivariant hypersphere output, and $\textbf{\textit{x}}\in \mathbb{R}^n$ be the input to the procedure~\ref{eq:algorithm}. 
We denote the output of the procedure~\ref{eq:algorithm} as $\textbf{F}(\textbf{x})$, where $\textbf{F}$ is the composition of all operations in the procedure~\ref{eq:algorithm}.
Since each operation is equivariant, \eqref{eq:equivariance} holds for each operation $\boldsymbol{\Phi}$, \ie, we have $\boldsymbol{\Phi}_i(\rho_i (\textbf{\textit{R}})\textbf{\textit{X}}) = \rho_{i+1}(\textbf{\textit{R}})\boldsymbol{\Phi}(\textbf{\textit{X}})$. 
Consider now the output $\textbf{{F}}(\textbf{x})$ and the transformed output $\textbf{F}(\textbf{\textit{R}}\textbf{x})$. Since each operation in $\textbf{F}$ is equivariant, we have:
\noindent
$\begin{aligned}
\textbf{F}(\textbf{\textit{R}}\textbf{x}) = \boldsymbol{\Phi}_d(\boldsymbol{\Phi}_{d-1}(\ldots\boldsymbol{\Phi}_2(\boldsymbol{\Phi}_1(\textbf{\textit{R}}\textbf{x})))) = \rho_d(\textbf{\textit{R}}) \boldsymbol{\Phi}_d(\boldsymbol{\Phi}_{d-1}(\ldots \boldsymbol{\Phi}_2(\boldsymbol{\Phi}_1(\textbf{x})))) = \rho_d(\textbf{\textit{R}}) \textbf{F}(\textbf{x}).
\end{aligned}$
Thus, the output of the procedure in~\eqref{eq:algorithm} is equivariant, as desired.
\end{proof}

\section{Architecture details}
\label{sec:A_architectures}
\begin{figure}[H]
    \includegraphics[width=\linewidth]{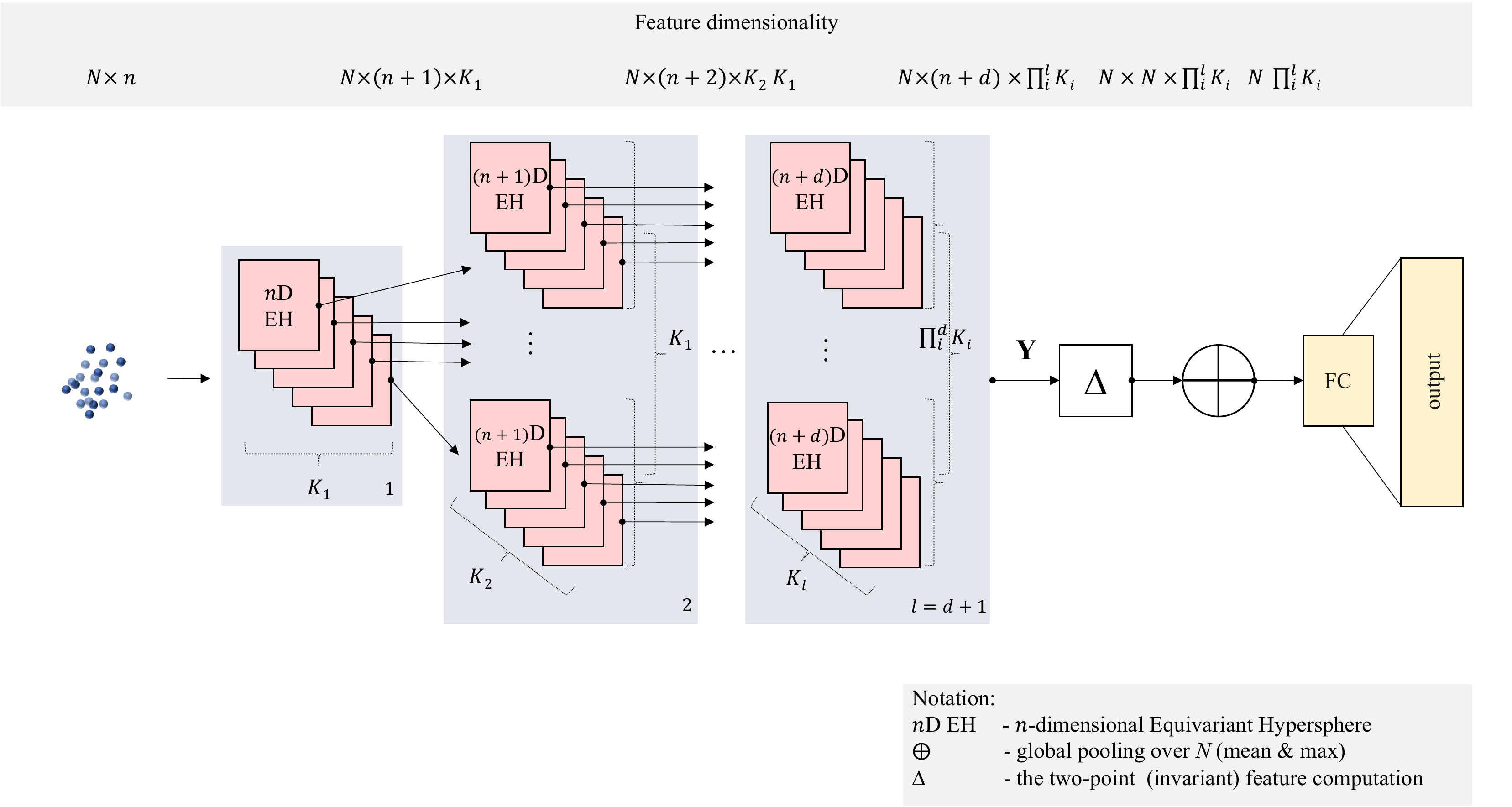}
    \caption{Architecture of our DEH model.
    All the operations are point-wise, \ie, shared amongst $N$ points. Each subsequent layer of equivariant hyperspheres contains $K_l$ neurons for each of the $\prod_i^{d} K_i$ preceding layer channels. The architectures of the non-permutation-invariant variants differ only in that the global aggregation function over $N$ is substituted with the flattening of the feature map.}
    \label{fig:deh_architecture}
\end{figure}
In this section, we provide an illustration of the architectures of our DEH model used in the experiments in Section~\ref{sec:demonstration}.
By default, we learned non-normalized hyperspheres and equipped the layers with the equivariant bias and the additional non-linearity (non-linear normalization in \eqref{eq:non-linear_norm}).
The number of learnable parameters corresponds to the competing methods in the experiments, as seen in Table~\ref{tab:num_parameters}.
The DEH architectures are presented in Table~\ref{tab:architectures}.
\begin{table*}
    \centering
    \begin{tabularx}{\linewidth}{@{\extracolsep{\fill}}lcccc}
        \toprule
        Methods & $\Og(3)$ Action recognition & $\Og(5)$ Regression & $\Og(5)$ Convex hulls \\
        \midrule
         CGENN & 9.1K & 467 $^\textbf{*}$& 58.8K  \\
        \midrule
         VN & 8.3K & 924 & N/A$^{\textbf{**}}$ \\
        \midrule
         GVP & 8.8K & 315 & N/A$^{\textbf{**}}$  \\
        \midrule
         Scalars & - & 641 & - \\
         \midrule
         SGM & 8.5K & 333 & 58.9K \\
        \midrule
         MLP &  8.3K & 447K \citep{finzi2021practical} & N/A$^{\textbf{**}}$ | 58.2K \\
        \midrule
        \textbf{ DEH (Ours)} & \textbf{8.1K} & \textbf{275} & \textbf{49.8K} \\
        \bottomrule
    \end{tabularx}
    \caption{Total number of parameters of the models in the experiments presented in Figure~\ref{fig:O5-experiments}. $^\textbf{*}$see Section~\ref{sec:A_o5_regression}. $^{\textbf{**}}$an unknown exact number of parameters, somewhere in the range of the other numbers in the column, as indicated by \citet{ruhe2023clifford}.}
    \label{tab:num_parameters}
\end{table*}

\begin{table*}
    \centering
    \begin{tabularx}{\linewidth}{@{\extracolsep{\fill}}lccccc}
        \toprule
        Methods & Equiv. layer sizes [$K_1$, $K_2$, \dots] & Invariant operation & FC-layer size & Total \#params & Performance \\
        \midrule
        \multicolumn{5}{c}{$\Og(3)$ Action recognition} & Acc., \% ($\uparrow$) \\
        \midrule
         DEH & [8, 6, 2] & sum & 32 & 7.8K & 69.86\\
         DEH & [3, 2] & $\Delta_E$ & 32 & 8.1K & 82.92 \\
         \textbf{DEH} & [3, 2] & $\Delta$ & 32 & 8.1K & \textbf{87.36} \\
        \midrule
        \multicolumn{5}{c}{$\Og(5)$ Regression} & MSE ($\downarrow$) \\
        \midrule
         DEH & [2] &  $l^2\text{-norm}$ & 32 & 343 & 0.0084 \\
         DEH & [2] & $\Delta_E$ & 32 & 275 &  0.0033 \\
         \textbf{DEH} & [2] & $\Delta$ & 32 & 275 &  \textbf{0.0007} \\
        \midrule
        \multicolumn{5}{c}{$\Og(5)$ Convex hulls} & MSE ($\downarrow$) \\
        \midrule
         DEH & [32, 24] & $l^2\text{-norm}$ & 32 & 57.2K &  7.5398\\
         DEH & [8, 6] & $\Delta_E$ & 32 & 49.8K & 1.3843 \\
         \textbf{DEH} & [8, 6] & $\Delta$ & 32 & 49.8K & \textbf{1.3166}\\
        \bottomrule
    \end{tabularx}
    \caption{Our DEH model architectures employed in the experiments and ablation on the invariant feature computation.
    The models are trained on all the available training data.
    The results of the models in \textbf{bold} are presented in Figure~\ref{fig:O5-experiments}.
    DEH for the first and the third tasks is also permutation-invariant.}
    \label{tab:architectures}
\end{table*}

\subsection{$\Og(5)$ Regression architectures clarification}
\label{sec:A_o5_regression}
Note that we used the CGENN model architecture, containing 467 parameters, from the first version of the manuscript \citep{ruhe2023clifford}, and the corrected evaluation protocol from the latest version.
Their model in the latest version has three orders of magnitude more parameters, which is in the range of the EMLPs \citep{finzi2021practical} (see Figure~\ref{fig:O5-experiments}, center) containing 562K parameters.
However, the error reduction thus achieved is only of one order of magnitude \citep{ruhe2023clifford} and only in the maximum training data size regime, which is why we compared the models within the original size range (see Table~\ref{tab:architectures} and Figure~\ref{fig:O5-experiments}).

Besides, since the number of points in this task is only 2 and the permutation invariance is not required (no aggregation over $N$; see Figure~\ref{fig:deh_architecture} and the caption), we used only three out of four entries of $\Delta$ \eqref{eq:Delta} in our model, \ie only one of the identical off-diagonal elements.
Also, we disabled the bias component in our model for this experiment and achieved a lower error ($0.0007$ vs. $0.0011$).

\section{Ablations}
\label{sec:A_ablations}
\subsection{Invariant feature computation} In Table~\ref{tab:architectures}, we show the effectiveness of the invariant operator $\Delta$ \eqref{eq:Delta}, modeling the relation between two points and a sphere (see Section~\ref{sec:higher_order}), over other invariant operations such as sum or $l^2$-norm, applied to the $N \times (n+d)$ matrix $\textbf{Y}$ (see Section~\ref{sec:deep_propagation} and Figure~\ref{fig:deh_architecture}) row-wise.

We also considered including the edge computation in $\Delta$, as discussed in Section~\ref{sec:higher_order}, in the following way:
\begin{equation}
\label{eq:Delta_E}
    \Delta_E = \textbf{E} \odot \textbf{Y} \, \textbf{Y}^{\top},
\end{equation}
where $\textbf{E} := \frac{1}{2} \, ({\lVert \textbf{x}_{i} - \textbf{x}_{j}\rVert}^2 + \textbf{I}_N ) \,\, \in \mathbb{R}^{N \times N}$ models the edges as the squared distances between the points (with the identity matrix included to also model interactions between a single point and a sphere). This formulation is slightly closer to the original formulation from \citet{li2001generalized} than \eqref{eq:Delta} that we used.
In Table~\ref{tab:architectures}, we present the respective model results.
\subsection{Architecture}
In Table~\ref{tab:multilayer}, we present a comparison between a single- vs. a two-layer DEH (which was employed in the experiments with the results in Figure~\ref{fig:main_figure}).
We note that already with one layer, our model exhibits high performance for the presented tasks.
Increasing the number of layers in our DEH is therefore only marginally advantageous in these cases.
\begin{table*}
    \centering
    \begin{tabularx}{\linewidth}{@{\extracolsep{\fill}}lccc}
        \toprule
        Methods & Equiv. layer sizes [$K_1$, $K_2$, \dots] & Total \#params & Avg.$^*$ performance \\
        \midrule
        \multicolumn{3}{c}{$\Og(3)$ Action recognition} & Acc., \% ($\uparrow$) \\
        \midrule
         DEH & [6] &  8.1K &  70.05 \\
         \textbf{DEH} & [3, 2] & 8.1K & \textbf{70.84} \\
        \midrule
        \multicolumn{3}{c}{$\Og(5)$ Convex hulls} & MSE ($\downarrow$) \\
        \midrule
         DEH & [48] & 49.6K &  2.1633 \\
         \textbf{DEH} & [8, 6] & 49.8K & \textbf{2.1024}\\
        \bottomrule
    \end{tabularx}
    \caption{Our DEH model: single- and two-layer (the results of which are presented in Figure~\ref{fig:O5-experiments}). $^*$The performance is averaged across the models trained on the \textit{various training set sizes} (see Figure~\ref{fig:O5-experiments}).}
    \label{tab:multilayer}
\end{table*}

Bias and learnable normalization ablations are presented in Table~\ref{tab:hyperparameters_ablation}. As we see, the performance of DEH is further improved if the bias is removed, which was also noted in Section~\ref{sec:A_o5_regression}.
A minor improvement is obtained by removing the learnable parameters from the non-linear normalization, $\alpha$ (one per neuron), \eqref{eq:non-linear_norm}, while keeping the bias. 
However, removing both the bias and the learnable parameters from the normalization, results in lower performance.

\begin{table*}
    \centering
    \begin{tabularx}{\linewidth}{@{\extracolsep{\fill}}ccccc}
        \toprule
        Bias & Normalization with learnable parameters, $\alpha_k$ & Total \#params & Avg.$^*$ MSE ($\downarrow$) \\
        \midrule
         \cmark & \cmark & 49.8K  &  2.1024 \\
         \cmark & \xmark & 49.7K  &  2.0936 \\
         \xmark & \cmark & 49.7K  &  \textbf{2.0623}\\
         \xmark & \xmark & 49.7K  &  2.1468 \\
        \bottomrule
    \end{tabularx}
    \caption{Hyperparameter ablation (using the $\Og(5)$ convex hull volume prediction task from Section~\ref{sec:demonstration}: our main DEH model with and without equivariant bias, learnable parameters in the normalization \eqref{eq:non-linear_norm}. $^*$The MSE is averaged across the models trained on the \textit{various training set sizes} (see Figure~\ref{fig:main_figure}).}
    \label{tab:hyperparameters_ablation}
\end{table*}

\end{document}